\documentclass[twoside,a4paper]{article}

\newlength\titlepageskip

\long\def\mydrafttitle#1{\title{\vspace*{-\titlepageskip}\hfill\fbox{\small Draft}\\[\titlepageskip]#1}}

\usepackage[utf8]{inputenc}
\usepackage[T1]{fontenc}
\usepackage{mathtools, amsfonts, amssymb, amsthm, biblatex}
\usepackage{tikz, tikz-3dplot}
\usepackage{hyperref}
\let\svthefootnote\thefootnote
\newcommand\freefootnote[1]{%
  \let\thefootnote\relax%
  \footnotetext{#1}%
  \let\thefootnote\svthefootnote%
}
\newtheorem{definition}{Definition}

\newtheorem{proposition}{Proposition}
\newtheorem{theorem}{Theorem}

\newtheorem*{remark}{Remark}

\newtheorem*{observation}{Observation}

\bibliography{biblio.bib}

\mydrafttitle{Hitting time for Markov decision process}
\author{Ruihcao Jiang, Javad Tavakoli, Yiqiang Zhao}

\begin{document}
\maketitle
\begin{abstract}
    We define the hitting time for a Markov decision process (MDP). We do not use the hitting time of the Markov process induced by the MDP because the induced chain may not have a stationary distribution. Even it has a stationary distribution, the stationary distribution may not coincide with the (normalized) occupancy measure of the MDP. We observe a relationship between the MDP and the PageRank. Using this observation, we construct an MP whose stationary distribution coincides with the normalized occupancy measure of the MDP and we define the hitting time of the MDP as the hitting time of the associated MP.
\end{abstract}
\section{Introduction}
\freefootnote{The first version of this paper pointed out some issues in the old version of \cite{fickinger}. The authors of \cite{fickinger} have then addressed these issues according to our suggestions in a new version. We therefore updated our paper, in which we removed contents related to these issues.}In \cite{fickinger}, Gromov-Wasserstein distance was used to compare different MDPs. In their approach, a distance function is needed on the space of state-action pairs of the underlying MDP. They used Euclidean distance. However, it is not entirely clear how to define the distance between actions. Moreover, it is desirable that the distance reflects some useful information about the MDP. For this reason, we propose to use the first-hitting time quasi-distance in places of the distance function. The rationale is as follows.
\begin{itemize}
    \item The occupancy measure to the MDP is whats the stationary distribution to an Markov process (MP);
    \item The stationary distribution of an MP is simply the inverse of the mean of the first-return time;
    \item The first-return time is simply the diagonal element for the first-hitting time matrix.
\end{itemize}
\section{Hitting time of MDP}
To define the first-hitting time of MDP, we first look at the occupancy measure of MDP.
\begin{definition}[Occupancy measure]
Let $(S,A,R,P,\gamma)$ be an \textup{MDP} with the initial distribution $\rho_0(s)$ on $S$ and the stationary policy $\pi$, the occupancy measure $\rho:S\times A\to \mathbb{R}$ is defined as follows.
\begin{equation}
    \rho_\pi(s,a)\coloneqq\sum_{t=0}^\infty\gamma^t\mathbb{P}\left(s_t=s,a_t=a|s_0\sim\rho_0,\pi\right).
\end{equation}
\end{definition}
We give a new interpretation of the occupancy measure. First, Tulcea's theorem (Proposition C.10 in \cite{hernandez-lerma}) says that a policy $\pi$ induces a Markov Process (MP) on $S\times A$ with the transition probability
\begin{equation}
    \label{induced_transition}
    P_{\pi}(s',a'|s,a)\coloneqq P(s'|s,a)\pi(a'|s')
\end{equation}
and the initial distribution $\rho_0(s,a)\coloneqq\rho_0(s)\pi(a|s)$. This describes a process as follows: Given the state-action pair $(s,a)$, transit to the next state $s'$ according to the transition probability of the original MDP; then immediately choose an action $a'$ according to the policy $\pi$.

Note that $\rho_\pi$ is not the stationary distribution (if it exists at all) of $(S\times A,P_\pi)$. First of all, $\rho_\pi$ is not a probability measure as it sums to $\frac{1}{1-\gamma}$. Even after a normalization, $\left(1-\gamma\right)\rho_\pi$ may differ from the stationary distribution. Consider a scenario where a car moves straight on a straight road and if it turns left or right, it will go off the road. Then none of the state-action pairs related to steering will be occupied and hence $\rho_\pi$ is zero on those pairs. However, for any stationary distribution, each component must be strictly positive. We will see later that it does not matter whether $(S\times A,P_\pi)$ is ergodic or not. Treating $\rho_\pi$ and $\rho_0$ as $1\times|S||A|$ vectors and $P_\pi$ as an $|S||A|\times|S||A|$ matrix, the following equation is well-known.
\begin{equation}\label{occupancy-measure}
    \rho_\pi=\left(I-\gamma P_\pi\right)^{-1}\rho_0.
\end{equation}
From Equation (\ref{occupancy-measure}), we make the following key observation.
\begin{observation}[PageRank]\label{ob1}
    $(1-\gamma)\rho_\pi$ coincides with the personalized PageRank vector on $S\times A$ with transition matrix $P_\pi$ and with probability $1-\gamma$ to restart with the initial distribution $\rho_0$.
\end{observation}
Second, we look at the first-hitting time of an MP.
\begin{definition}[First-hitting time of MP]
    Let $\{X_t\}$ be an MP. The first-hitting time $T$ is the random variable defined as follows.
    \begin{equation}
        T_{ij}=\begin{cases}
            \inf_{t\in\mathbb{N}}\{t\ |\ X_0=j,X_1\neq i,\cdots,X_{t-1}\neq i,X_t=i\} &\text{if}\ i\neq j\\
            0, & \text{otherwise.}
        \end{cases}
    \end{equation}
    If $i$ is never reached from $j$, $T_{ij}=+\infty$, where we have adopted the convention that $\inf\emptyset=+\infty$.
\end{definition}
By abuse of the notation, we call the expectation $\mathbb{E}T_{ij}$ of the random variable $T_{ij}$ also the first-hitting time and denote it by $T_{ij}$.
\begin{equation}
    T_{ij}\coloneqq\mathbb{E}T_{ij}=\begin{cases}
        \sum_{t=1}^\infty t\mathbb{P}(T_{ij}=t) &\text{if}\ i\neq j\\
        0 &\text{otherwise}.
    \end{cases}
\end{equation}
If the MP has transition matrix $P_\pi$, a one-step analysis shows that $T_{ij}$ satisfies the following recursive relation.
\begin{equation}
    \label{mean_hitting_time}
    T_{ij}=\begin{cases}
        1+\sum_kT_{ik}P_\pi(k|j) &\text{if}\ i\neq j\\
        0 &\text{otherwise}.
    \end{cases}
\end{equation}
By our Observation, we reduce the task of defining a first-hitting time for an MDP $(S,A,R,P,\gamma)$ to that of defining a first-hitting time for a PageRank $(S\times A, P_\pi,1-\gamma)$. A PageRank can be described equivalently by a new process $(S\times A,\tilde{P}_\pi)$ with the transition matrix
\begin{equation}
    \label{modified_transition}
    \tilde{P}_\pi(x'|x)\coloneqq(1-\gamma)\rho_0(x')+\gamma P_\pi(x'|x).
\end{equation}
$(S\times A,\tilde{P}_\pi)$ is obviously Markovian. The initial condition for this new process can be arbitrary, in particular, not necessarily $\rho_0$ since it turns out that $(S\times A,\tilde{P}_\pi)$ is ergodic and its stationary distribution is given by $(1-\gamma)\rho_\pi$ (See Proposition 1 and Corollary 1 of \cite{avrachenkov}). The ergodicity of $(S\times A,\tilde{P}_\pi)$ implies that the first-hitting time is finite on supp$(\rho_\pi)$. In this case, Equation (\ref{mean_hitting_time}) reads
\begin{equation}
    \label{hitting_time}
    T_{ij}=\begin{cases}
        1+(1-\gamma)\sum_kT_{ik}\rho_0(k)+\gamma\sum_kT_{ik}P_\pi(k|j) &\text{if}\ i\neq j\\
        0 &\text{otherwise}.
    \end{cases}
\end{equation}
There is also a related quantity $L_{ij}$ satisfying the following recursive relation.
\begin{equation}
    L_{ij}=\begin{cases}
        1+\gamma\sum_kL_{ik}P_\pi(k|j) &\text{if}\ i\neq j\\
        0 &\text{otherwise}.
    \end{cases}
\end{equation}
The probabilistic interpretation of $L_{ij}$ is the expectation of the discounted path length from $j$ to $i$ for the first time. $L_{ij}$ and $T_{ij}$ are related by the following formula (Theorem 1(b) of \cite{avrachenkov}).
\begin{equation}
    T_{ij}=\frac{L_{ij}}{1-(1-\gamma)\sum_{k}L_{ik}\rho_0(k)}.
\end{equation}
\begin{remark}
    $L_{ij}$ is called the (1-$\gamma$)-discounted hitting time in \cite{sarkar} although $L_{ij}$ is not the first-hitting time of $(S\times A,P_\pi)$, whose first-hitting time is given by Equation (\ref{mean_hitting_time}), nor of $(S\times A,\tilde{P}_\pi)$, whose first-hitting time is given by Equation (\ref{hitting_time}). However, $L_{ij}$ is easier to estimate empirically.
\end{remark}
We have shown that
\begin{itemize}
    \item The occupancy measure of an MDP $(S, A, R, P, \gamma)$ with the initial distribution $\rho_0$ coincides with the PageRank vector of the personalized PageRank $(S\times A, P_\pi, 1-\gamma)$ with the initial distribution $\rho_0$.
    \item The PageRank vector of the personalized PageRank $(S\times A, P_\pi, 1-\gamma)$ with the initial distribution $\rho_0$ coincides with the stationary distribution of the ergodic MP $(S\times A, \tilde{P}_\pi)$ with arbitrary initial distribution.
\end{itemize}
Based on the above facts, we propose the following definition.
\begin{definition}[First-hitting time of MDP]
Let $(S,A,R,P,\gamma)$ be an \textup{MDP} with a stationary policy $\pi$. Let $P_\pi$ be the induced transition matrix in Equation (\ref{induced_transition}) and $\tilde{P}_\pi$ be the modified transition matrix in Equation (\ref{modified_transition}). The first-hitting time of the \textup{MDP} is defined as the first-hitting time of $(S\times A,\tilde{P}_\pi)$ with its expectation given by Equation (\ref{hitting_time}).
\end{definition}
Since we define the first-hitting time of MDP as the first-hitting time of some MP, we immediately have the following.
\begin{proposition}
The first-hitting time of an \textup{MDP} gives rise to a quasi-distance $T_\pi(i,j)\coloneqq T_{ij}$, i.e. a distance function without satisfying the symmetry condition.
\end{proposition}
\section{Hitting time quasidistance induces distance}
It might be a concern that we use a quasi-distance, not a distance. In this section, we show that applying $\mathcal{GW}$ to a quasi-metric measure spaces still yield a distance (our Theorem 2).

We have constructed a triple $(S\times A,\rho_\pi,T_\pi)$ from the \textup{MDP} $(S,A,R,P,\gamma)$. Now we can apply the $\mathcal{GW}$ construction \cite{memoli} to define a quantity between two MDPs
\begin{definition}
    Let $(S_1,A_1,R_1,P_1,\gamma_1)$ and $(S_2,A_2,R_2,P_2,\gamma_2)$ be MDPs. Denote by $(\mathcal{X},\rho_\mathcal{X},T_\mathcal{X})$ and $(\mathcal{Y},\rho_\mathcal{Y},T_\mathcal{Y})$ their induced triples, respectively. Then
    \begin{equation}
        \label{gw}
        \begin{split}
            &\mathcal{GW}((S_1,A_1,R_1,P_1,\gamma_1),(S_2,A_2,R_2,P_2,\gamma_2))\\
            \coloneqq&\mathcal{GW}((\mathcal{X},(1-\gamma_1)\rho_\mathcal{X},T_\mathcal{X}),(\mathcal{Y},(1-\gamma_2)\rho_\mathcal{Y},T_\mathcal{Y}))\\
            =&\min_{\mu}\frac{1}{2}\left[\sum_{\mathcal{X}\times\mathcal{Y}}\sum_{\mathcal{X}\times\mathcal{Y}}|T_\mathcal{X}(x,x')-T_\mathcal{Y}(y,y')|^2\mu(x,y)\mu(x'.y')\right]^{\frac{1}{2}},
        \end{split}
    \end{equation}
    where $\mu\in\mathcal{M}((1-\gamma_1)\rho_\mathcal{X},(1-\gamma_2)\rho_\mathcal{Y})$ is the set of couplings, i.e.
    \begin{equation}
        \label{boundary_condition}
        \mathcal{M}((1-\gamma_1)\rho_\mathcal{X},(1-\gamma_2)\rho_\mathcal{Y})=\left\{\mu\left|\sum_{y\in\mathcal{Y}}\mu=(1-\gamma_1)\rho_\mathcal{X}\ \text{and}\ \sum_{x\in\mathcal{X}}\mu=(1-\gamma_2)\rho_\mathcal{Y}\right\}\right.
    \end{equation}
\end{definition}
\begin{remark}
    It is necessary to normalize the occupancy measures. Otherwise, if $\gamma_1\neq\gamma_2$, then no coupling exists since
    \begin{equation*}
        \sum_{x\in\mathcal{X}}\sum_{y\in\mathcal{Y}}\mu=\sum_{x\in\mathcal{X}}\rho_\mathcal{X}=\frac{1}{1-\gamma_1}\neq\frac{1}{1-\gamma_2}=\sum_{y\in\mathcal{Y}}\rho_\mathcal{Y}=\sum_{x\in\mathcal{X}}\sum_{y\in\mathcal{Y}}\mu.
    \end{equation*}
\end{remark}
Different MDPs may give rise to the same triple. To eliminate this redundancy, we introduce an equivalent relation on $\mathcal{MDP}$, the set of all MDPs.
\begin{definition}[Equivalent relation on $\mathcal{MDP}$]
    $\textup{MDP}_1$ and $\textup{MDP}_2$ are said to be equivalent, denoted by $\textup{MDP}_1\sim\textup{MDP}_2$ if they give rise to the same triple, i.e. there exists a bijection $\varphi:S_1\times A_1\to S_2\times A_2$ s.t. $T_{1}(x,x')=T_{2}(\varphi(x),\varphi(x'))$ and $(1-\gamma_1)\rho_1(x)=(1-\gamma_2)\rho_2(\varphi(x))$ for all $x,x'\in S_1\times A_1$.
\end{definition}
We have not yet called $\mathcal{GW}(\textup{MDP}_1,\textup{MDP}_2)$ a distance. We know that $\mathcal{GW}$ is a distance function on the space of metric-measure spaces after quotient out redundancy \cite{memoli}. However, the first-hitting time is only a quasi-distance, hence $(S\times A,(1-\gamma)\rho_\pi,T_\pi)$ is not necessarily a metric-measure space. We prove that $\mathcal{GW}$ is indeed a distance function on $\mathcal{MDP}/\sim$.
\begin{theorem}\label{real_distance}
    $\mathcal{GW}$ is a distance on $\mathcal{MDP}/\sim$.
\end{theorem}
\begin{proof}
    Theorem 16 in \cite{chowdhury} says that if $T_\mathcal{X}$ and $T_\mathcal{Y}$ are measurable, then $\mathcal{GW}$ is a pseudo-distance. Hence,to show that $\mathcal{GW}$ is a distance, we need to show that $\mathcal{GW}(\textup{MDP}_1,\textup{MDP}_2)=0$ iff $\textup{MDP}_1\sim\textup{MDP}_2$. The ``if" part is trivial. We only need to show $\mathcal{GW}(\textup{MDP}_1,\textup{MDP}_2)=0\implies\textup{MDP}_1\sim\textup{MDP}_2$. Let $\mu$ be a coupling s.t. Equation (\ref{gw}) evaluates to $0$. Then $T_\mathcal{X}(x,x')=T_\mathcal{Y}(y,y')$ for all pairs of $(x,y)$ and $(x',y')$ s.t. $\mu(x,y)>0$, $\mu(x',y')>0$. Let $x=x'$. Then $T_\mathcal{X}(x,x)=T_\mathcal{Y}(y,y')=0$. Since $T_\mathcal{Y}$ is a quasi-distance, it implies that $y=y'$. Hence $\mu$ does not split mass i.e. $\mu$ induces a bijection $\varphi$ s.t. $T_\mathcal{X}(x,x')=T_\mathcal{Y}(\varphi(x),\varphi(x'))$. It remains to check measures. By the boundary condition Equation (\ref{boundary_condition}),
    \begin{equation*}
        \begin{split}
            (1-\gamma_2)\rho_{\mathcal{Y}}(\varphi(x))&=\sum_{x\in\mathcal{X}}\mu(x,\varphi(x))\\
            &=\sum_{\varphi(x)\in\mathcal{Y}}\mu(x,\varphi(x))\\
            &=(1-\gamma_1)\rho_\mathcal{X}(x)
        \end{split}
    \end{equation*}
    as desired.
\end{proof}
\begin{remark}
    Intuitively, the definition of $\mathcal{GW}$ involves a minimization over all couplings. If $T_\pi(x,x')\neq T_\pi(x',x)$, the minimization will choose the smaller one.
\end{remark}
\begin{remark}
    Theorem 16 in \cite{chowdhury} only says that their $\mathcal{GW}$ is a pseudo-distance. They explicitly constructs an example in their Remark 15 where their $\mathcal{GW}$ fails to be a distance. This is because they assume nothing about their network weight function $\omega$, except measurability. In particular, $\omega(x,x)$ in their Remark 15 does not evaluate to zero. However, for the first-hitting time, by definition $T(x,x)=0$.
\end{remark}
\section{Conclusion}
    We defined the hitting time for an MDP via a relationship between the MDP and the PageRank. The hitting time is a quasi-distance. By applying $\mathcal{GW}$ construction, the hitting time quasi-distance induces a distance on the set of discrete MDPs.
\printbibliography
\end{document}